\documentclass{article}
\pdfoutput=1

\usepackage{PRIMEarxiv}

\usepackage[utf8]{inputenc} % allow utf-8 input
\usepackage[T1]{fontenc}    % use 8-bit T1 fonts
\usepackage{url}            % simple URL typesetting
\usepackage{booktabs}       % professional-quality tables
\usepackage{amsfonts}       % blackboard math symbols
\usepackage{nicefrac}       % compact symbols for 1/2, etc.
\usepackage{microtype}      % microtypography
\usepackage{lipsum}
\usepackage{fancyhdr}       % header
\usepackage{graphicx}       % graphics
\graphicspath{{media/}}     % organize your images and other figures under media/ folder
\usepackage{amsmath}
\usepackage{mathtools}
\usepackage{amsthm}
\usepackage{hyperref}
\usepackage{amssymb}
\usepackage[mathscr]{euscript}
\usepackage{amssymb}
\usepackage{graphicx}
\usepackage{appendix}
\usepackage{epsfig}
\usepackage{algorithm}
\usepackage{algpseudocode}
\usepackage[capitalize,noabbrev]{cleveref}

\theoremstyle{plain}
\newtheorem{theorem}{Theorem}[section]
\newtheorem{proposition}[theorem]{Proposition}

\theoremstyle{definition}

\theoremstyle{remark}

\title{Streaming data recovery 
\\via Bayesian tensor train decomposition
}

\author{
  Yunyu Huang, Yani Feng, Qifeng Liao\thanks{Corresponding author} \\
  School of Information Science and Technology \\
  ShanghaiTech University \\
  Shanghai, 201210, China\\
  \texttt{\{edenhyy@foxmail.com, fengyn@shanghaitech.edu.cn, liaoqf@shanghaitech.edu.cn\}} \\
}

\begin{document}
\maketitle

\begin{abstract}
In this paper, we study a Bayesian tensor train (TT) decomposition method to recover streaming data by approximating the latent structure in high-order streaming data. Drawing on the streaming variational Bayes method, we introduce the TT format into Bayesian tensor decomposition methods for streaming data, and formulate posteriors of TT cores. Thanks to the Bayesian framework of the TT format, the proposed algorithm (SPTT) excels in recovering streaming data with high-order, incomplete, and noisy properties. The experiments in synthetic and real-world datasets show the accuracy of our method compared to state-of-the-art Bayesian tensor decomposition methods for streaming data.
\end{abstract}

% keywords can be removed
\keywords{Variantional inference \and  Tensor train decomposition \and Streaming tensor}

\section{Introduction}
\subsection{Motivation}
In recent decades, there has been tremendous interest in streaming data recovery across diverse domains such as recommendation systems~\cite{huang2015tencentrec}, sensor networks~\cite{hu2022streaming}, and social media analytics~\cite{zipkin2016point}. Streaming data, characterized by its continuous generation and real-time flow, is susceptible to corruption during transmission and storage, leading to considerable difficulties in data analysis. Tensor decomposition is an important tool to represent and recover tensor data by introducing its latent structure. In practical applications, the latent structure is discovered by minimizing the reconstruction error of noisy observed tensor elements and then can be used to predict missing elements.

Effective numerical techniques, such as CANDECOMP/PARAFAC (CP) decomposition \cite{hars1, bro1997parafac,TANG2020109326} Tucker decomposition \cite{tucker1966some,xiao2021efficient}, and tensor train (TT) decomposition \cite{oseledets2011tensor,feng2020tensor}, are frequently used for tensor decomposition and have been proposed to compress full tensors to obtain their latent representations, i.e., the latent structures. Although these techniques have achieved great success, they require processing on static data, meaning that the entire observed data is utilized at each iteration step. However, in practical scenarios such as log files from web application~\cite{roy2020survey} and information from social networks~\cite{koutra2012tensorsplat}, only partial observed elements are generated at each time step, and previously produced observed elements cannot be retrieved due to the stress on database capacity and privacy. Consequently, the techniques above will be unsuitable, prompting the necessity to update the latent structure online based on the current observed elements, and recover missing elements using the latent structure. 

\subsection{Related Work}
Using CP decomposition to recover streaming data originated in~\cite{nion2009adaptive}, while subsequent enhancements are proposed by \cite{minh2016adaptive,nguyen2016fast,nguyen2017second}. To analyze streaming data with multi-index structures commonly encountered in real scenarios, successive CP decomposition based algorithms like MAST \cite{song2017multi}, OR-MSTC \cite{najafi2019outlier}, InParTen2~\cite{yang2020multi}, and DisMASTD~\cite{yang2023dismastd} have emerged. In \cite{sun2006beyond,sun2008incremental}, a dynamic tensor analysis (DTA) algorithm is proposed to approximate latent structure of tucker decomposition. The Rapid Incremental Singular Value Decomposition (ISVD), introduced by \cite{li2007robust,hu2011incremental}, streamlines the solution process within the DTA framework.

Meanwhile, Bayesian inference offers a promising approach for tackling the challenge of tensor decomposition in streaming data recovery. The forefront Bayesian tensor decomposition algorithms for streaming data include POST~\cite{du2018probabilistic} and BASS-Tucker~\cite{fang2021bayesian}, grounded in the Bayesian formulations of CP and Tucker decomposition, respectively. While both POST and BASS-Tucker serve as robust streaming decomposition methods adaptable to various applications, two key challenges in modeling and computation for estimating the latent structure warrant attention. Firstly, the latent structures (CP and Tucker formats) in POST and BASS-Tucker struggle to effectively balance representation power and computational efficiency for high-order tensors. Secondly, in Bayesian tensor decomposition algorithms for streaming data, the design of a reliable prior for the latent structure is crucial for successful estimation.
\subsection{Contributions}
We introduce SPTT, a Streaming Probabilistic Tensor Train decomposition approach, for recovering streaming data by estimating the posterior of TT format. The contributions of our work are threefold:
\begin{itemize}
\item First, SPTT is founded on TT decomposition, combining the benefits of CP and Tucker decomposition. This stems from its provision of a space-saving model known as TT format while retaining strong representation capabilities.

\item Second, inspired by the Bayesian TT representation for static data in~\cite{9513808}, we devise a robust Gaussian prior of TT-cores tailored for streaming data. We employ the streaming variational Bayes (SVB) method on the Gaussian prior for streaming data analysis.

\item Third, both synthetic and real-world data are used to demonstrate the accuracy of our algorithm.
\end{itemize}

\section{Preliminaries}
\label{sec:ba}
Throughout the paper, we denote scalar, vector, matrix and tensor variables, respectively, by lowercase letters (e.g. $x$), boldface lowercase letters (e.g. $\mathbf{x}$), boldface capital letters (e.g. $\mathbf{X}$), and boldface Euler script letters (e.g. $\pmb{\mathscr{X}}$). For a $D$-th order tensor $\pmb{\mathscr{X}}\in \mathbb{R}^{N_{1} \times \cdots \times N_{D}}$, the $\mathbf{j}:=(j_{1},\cdots, j_{D})$-th element of $\pmb{\mathscr{X}}$ is defined by $x_{\mathbf{j}}$, where $j_{d} = 1,\ldots,N_{d}$ for $ d=1,\ldots,D$. 
\subsection{Tensor Train Decomposition}

Tensor decomposition introduces a latent structure to represent the each element of tensors. We aim to infer the latent structure from incomplete noisy observed tensors. Tensor train (TT) decomposition has a great representation ability, especially for high-order tensors. Given a $D$-th order tensor $\pmb{\mathscr{X}}$, the tensor train decomposition \cite{oseledets2011tensor} is
\begin{equation}\label{TTdecom}
    x_{\mathbf{j}} \approx \prod_{d=1}^{D} \pmb{\mathscr{G}}^{(d)}_{j_{d}},    
\end{equation}
where 
$\pmb{\mathscr{G}}^{(d)}_{j_{d}}\in \mathbb{R}^{r_{d-1}\times r_{d}}$ is the $j_{d}$-th slice of TT-cores $\pmb{\mathscr{G}}^{(d)} \in \mathbb{R}^{r_{d-1}\times N_{d}\times r_{d}}$, for $j_{d} = 1,\ldots,N_{d}$, $d = 1,\ldots,D$, and the ``boundary condition'' is $r_0 = r_D = 1$. The vector $[r_0,r_1,r_2, \ldots, r_D]$ is referred to as TT-ranks. The element-wise form \eqref{TTdecom} is considered to be in the TT-format shown in Figure \ref{ttformat_fig}. In other words, TT decomposition is to decompose a $D$-th order tensor $\pmb{\mathscr{X}}$ into a
sequence of three-way factor tensors. For convenience, a more compact form is given by
\begin{equation}
    \pmb{\mathscr{X}} \approx \langle\langle \pmb{\mathscr{G}}^{(1)},\ldots,\pmb{\mathscr{G}}^{(D)} \rangle\rangle.
\end{equation}
where $\langle\langle \cdot \rangle\rangle$ is a shorthand of multi-linear product \eqref{TTdecom} in TT decomposition. To infer the latent structure $\pmb{\mathscr{G}}=\{\pmb{\mathscr{G}}^{(d)}\}_{d=1}^{D}$ from noisy observed elements, one can minimize the following reconstruction error,
\begin{equation}\label{ttloss}
   \mathcal{L}(\pmb{\mathscr{G}}) =  \sum_{\mathbf{j}\in \Omega}\| x_{\mathbf{j}} -\prod_{d=1}^{D} \pmb{\mathscr{G}}^{(d)}_{j_{d}}\|^{2},
\end{equation}
where $\Omega$ represents the index set of the whole observed data. The alternating least squares algorithm~\cite{wang2016tensor} is commonly used to update each TT core alternatively, given all the other fixed.  

\begin{figure}[ht]
\vskip 0.2in
\begin{center}
\centerline{\includegraphics[width=\columnwidth]{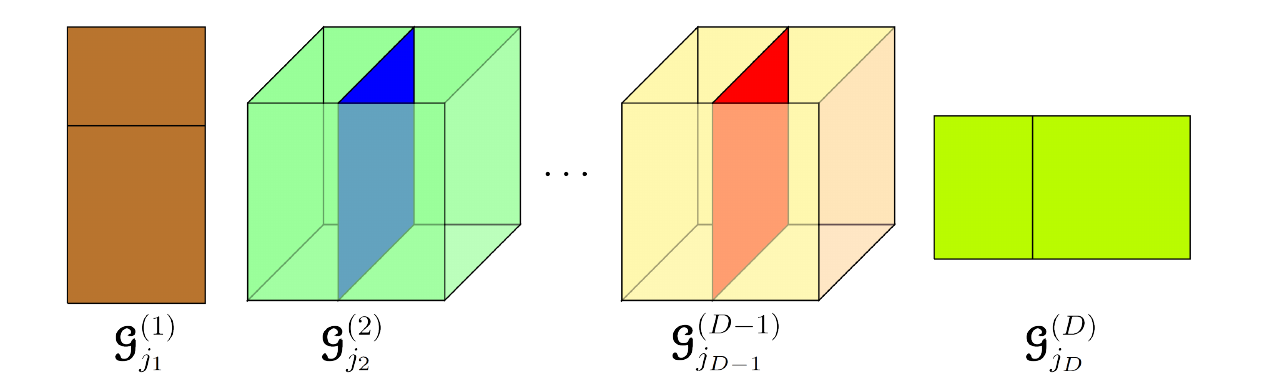}}
\caption{TT decomposition for an element $x_{\mathbf{j}}$ (TT-format).}
\label{ttformat_fig}
\end{center}
\vskip -0.2in
\end{figure}

\subsection{Streaming Tensor Train Decomposition Problem Setting}\label{setting}
%The standard TT decomposition, like~\cite{wang2016tensor,YUAN201953}, use the point estimation to approximate TT-cores and is not capable of evaluating the uncertainty, which can vary among each slice in a TT core
The problem \eqref{ttloss} need to conduct on static data, which implies that at each iteration step, one need to have access to the whole observed data. However, in practical application the observed data is generated dynamically. Due to the stress on database capacity and privacy, each time one only have access to the current observed data and cannot revisit the previously generated observed data. Moreover, traditional TT decomposition methods like~\cite{wang2016tensor, YUAN201953} use point estimation \eqref{ttloss} to estimate TT-cores and is not capable of evaluating uncertainties, which can vary among each slice in TT-cores. To address the these issues, a new approach is developed for solving the streaming tensor decomposition problem and also estimating uncertainties. 

Our goal is to infer TT-cores $\{\pmb{\mathscr{G}}^{(d)}\}_{d=1}^{D}$ from streaming data $\{B_{1}, B_{2},\ldots\}\subset B$, where $B_{t} = \{x_{\mathbf{j}}|\mathbf{j}\in\Omega_{t}\}$ and $\Omega_{t}$ denotes its index set, but we can only visit data batch $B_{t}$ at time $t$, instead of revisiting previous data $\Tilde{B} = \{B_{1},\ldots, B_{t-1}\}$.
The streaming variational Bayes (SVB)~\cite{broderick2013streaming} method provides a Bayesian framework for the streaming TT decomposition problem, where uncertainties can be taken into consideration.
Thus, we aim to find the posterior distribution of latent parameters $\mathbf{\theta}$ including TT-cores from recursively received data. That is, to estimate the density function $p(\mathbf{\theta}|\tilde{B},B_{t})$ with the current likelihood $p({B}_{t}|\mathbf{\theta})$ and updated density function $p(\mathbf{\theta}|\tilde{B})$ using the Bayes' rule
%the proportional relationship in the SVB method
as follows,
\begin{equation}\label{eq:stream}
	p(\mathbf{\theta}|\Tilde{B}, B_{t}) \propto p(\mathbf{\theta}|\Tilde{B})p(B_{t}|\mathbf{\theta}).
\end{equation}
%{\color{red} this is bayes' rule not SVB method, and then give more details of SVB}

In \cref{bayesianTT}, we will introduce a new probabilistic TT model, where the latent parameters is considered as random variables. Applying the SVB method, the recursively updating process \eqref{eq:stream} based on the probabilistic TT model can be achieved in \cref{BSI}.

\section{Bayesian model for streaming tensor train decomposition}\label{sec:sptt}
In this section, we first develop the probabilistic model of TT decomposition, where the Gaussian noise is included. After that, the posterior inference for TT-cores is conducted by the SVB method.
Finally, we present the streaming probabilistic tensor train decomposition (SPTT) algorithm.
\subsection{Probabilistic Modeling of Tensor Train Decomposition}\label{bayesianTT}
 Inspired by a Bayesian TT representation for static data in~\cite{9513808}, we introduce a Bayesian TT model for streaming data. We first specify $ \{\pmb{\mathscr{G}}^{(d)}\}_{d=1}^{D}$ from a Gaussian prior,
\begin{equation}
\label{eq:tt_prior}
p(\{\pmb{\mathscr{G}}^{(d)}\}_{d=1}^{D}) = \prod_{d=1}^{D}\prod_{k=1}^{r_{d-1}}\prod_{l=1}^{r_{d}}\mathcal{N}\left(\pmb{\mathscr{G}}^{(d)}_{k,:,l}\bigg|\mathbf{m}_{\pmb{\mathscr{G}}_{k,:,l}^{(d)}},v\mathbf{I}\right),
\end{equation}
 which indicates that each fiber  $\pmb{\mathscr{G}}^{(d)}_{k,:,l}$ admits a Gaussian distribution with mean $\mathbf{m}_{\pmb{\mathscr{G}}_{k,:,l}^{(d)}}$, and covariance matrix $v\mathbf{I}$. $v$ is a scalar and controls the flatness of the Gaussian prior. In streaming TT decomposition problem, the prior in \eqref{eq:tt_prior} is more reliable than the prior in~\cite{9513808}, because the latter can fail to estimate TT-cores through a small amount of data by inducing TT-cores with all zeros elements.
 
 Given TT-cores $ \{\pmb{\mathscr{G}}^{(d)}\}_{d=1}^{D}$, the likelihood of an observed element $x_{\mathbf{j}}$ is 
\begin{equation}\label{eq:likelihood}p(x_{\mathbf{j}}|\{\pmb{\mathscr{G}}^{(d)}\}_{d=1}^{D},\tau) = \mathcal{N}\left(x_{\mathbf{j}}\big|\prod_{d=1}^{D} \pmb{\mathscr{G}}^{(d)}_{j_{d}},\tau^{-1}\right),\end{equation}
where $\tau$ is the noise precision. Due to the non-negative and long tail, the Gamma distribution provides a good model for $\tau$. We assign a Gamma prior over $\tau$,
\begin{equation}
\label{eq:tau}p(\tau|\alpha_{0},\beta_{0}) = \operatorname{Gam}(\tau|\alpha_{0},\beta_{0}),
\end{equation}
where $\alpha_{0}$ and $\beta_{0}$ are hyperparameters  of the Gamma distribution. 
%The density function of a Gamma distribution is
%\begin{equation*}
%p(x|\alpha,\beta) = \frac{x^{\alpha-1}e^{-\beta x}\beta^{\alpha}}{\Gamma(\alpha)},
%\end{equation*}
%where $\Gamma(\alpha)$ is a Gamma function. 
The joint probability is then given by combining \eqref{eq:tt_prior} with \eqref{eq:likelihood}-\eqref{eq:tau},
\begin{equation}\label{eq:joint}\begin{aligned}
p\left(x_{\mathbf{j}\in\Omega},\{\pmb{\mathscr{G}}^{(d)}\}_{d=1}^{D},\tau\right)& = \operatorname{Gam}(\tau|\alpha_{0},\beta_{0})\\     &\cdot \prod_{d=1}^{D}\prod_{k=1}^{r_{d-1}}\prod_{l=1}^{r_{d}} \mathcal{N}\left(\pmb{\mathscr{G}}^{(d)}_{k,:,l}\bigg|\mathbf{m}_{\pmb{\mathscr{G}}_{k,:,l}^{(d)}},v\mathbf{I}\right)\\     &\cdot\prod_{\mathbf{j}\in\Omega}\mathcal{N}\left(x_{\mathbf{j}}\big|\prod_{d=1}^{D} \pmb{\mathscr{G}}^{(d)}_{j_{d}},\tau^{-1}\right).
\end{aligned}\end{equation}

\subsection{Bayesian Streaming Inference}
\label{BSI}
We propose a streaming probabilistic TT decomposition approach (SPTT) for streaming data,
%introduce SPTT, our posterior inference algorithm for streaming data,
based on TT decomposition in section \ref{bayesianTT} and the streaming variational Bayes (SVB)~\cite{broderick2013streaming}. For streaming data $\{B_{1},B_{2},\ldots\}$ in \cref{setting}, our goal is to conduct the posterior inference for each parameter in $\mathbf{\theta} =\{ \{\pmb{\mathscr{G}}^{(d)}\}_{d=1}^{D}, \tau\}$ upon receiving each data batch $B_{t}$, without revisiting the previous batches $\Tilde{B} = \{B_{1},\ldots,B_{t-1}\}$.

%Assume the prior distribution of $ \mathbf{\theta}$ is $p(\mathbf{\theta})$.

%Assume the prior distribution of $ \mathbf{\theta}=\{\{\pmb{\mathscr{G}}^{(d)}\}_{d=1}^{D},\tau\}$ is $p(\mathbf{\theta})$. Denote $\Omega_{t}$ is the index set of $t$-th data batch $B_{t}$. For the first data batch $B_{1}=\{x_{\mathbf{j}}\}_{\mathbf{j}\in \Omega_{1}}$, the posterior $p(\mathbf{\theta}|B_{1})$ can be calculated via Bayes' rule:
%\begin{equation}
%    p(\mathbf{\theta}|B_{1}) \propto p(\mathbf{\theta})p(B_{1}|\mathbf{\theta}),
%\end{equation}
%where $p(B_{1}|\mathbf{\theta}) = \prod_{\mathbf{j}\in \Omega_{1}}p(x_\mathbf{j}|\mathbf{\theta})$. Similarly,
%%Suppose the data streams are provided in a series of data batch $ \{B_{1},B_{2},\ldots\}$, corresponding to index set $\{\Omega_{1},\Omega_{2},\ldots\}$, 
%when receiving the newly data batch $B_t=\{x_{\mathbf{j}}\}_{\mathbf{j}\in \Omega_{t}}$, the incremental posterior is given by
%\begin{equation} \label{eq:stream}
%    p(\mathbf{\theta}|\Tilde{B}\cup B_{t}) \propto p(\mathbf{\theta}|\Tilde{B})p(B_{t}|\mathbf{\theta}),
%\end{equation}
The proportional relationship in \eqref{eq:stream} motivates us to conduct a recursive process of streaming inference for $\mathbf{\theta}$. However, the multilinear product of TT-cores $\{\pmb{\mathscr{G}}^{(d)}\}_{d=1}^{D}$ in \eqref{eq:likelihood} results in the posterior update intractable. To overcome the problem, we introduce a factorized variational posterior distribution
% \begin{equation}\label{factorization}
  $  q^{(t-1)}(\mathbf{\theta}) = \prod_{i} q^{(t-1)}_{\mathbf{\theta}_{i}}(\mathbf{\theta}_{i})$
% \end{equation}
 to approximate $p(\mathbf{\theta}|\Tilde{B})$ in the right hand-side of \eqref{eq:stream}, where $\mathbf{\theta}_{i}$ indicates a component in $\mathbf{\theta}$. With receiving new data batch $B_{t}$, we can obtain a new distribution 
 \begin{equation}\label{eq:post_app}
 	  \Tilde{p}^{(t)}(\mathbf{\theta}) = q^{(t-1)}(\mathbf{\theta})p(B_{t}|\mathbf{\theta}),
 \end{equation}
which is the product of posterior after $(t-1)$-th data batch and the current likelihood of $B_{t}$ and can be regarded as an approximation of the joint distribution $p(\mathbf{\theta},\Tilde{B}, B_{t})$ up to a normalization constant. More details for computing $\log \Tilde{p}^{(t)}(\mathbf{\theta})$ can be found in Appendix \ref{appendix1}. Moreover, the current posterior $p(\mathbf{\theta}|\Tilde{B}, B_{t})$ in left hand-side of \eqref{eq:stream} can be approximated by $q^{(t)}(\mathbf{\theta})$, 
\begin{equation}\label{eq:mean_field1}
p(\mathbf{\theta}|\Tilde{B}, B_{t})\approx q^{(t)}(\mathbf{\theta}) = q^{(t)}_{\tau}(\tau)\prod_{d=1}^{D}\prod_{k=1}^{r_{d-1}}\prod_{l=1}^{r_{d}}\prod_{j_{d}}^{N_{d}}q^{(t)}_{\pmb{\mathscr{G}}^{(d)}_{k,j_{d},l}}\left(\pmb{\mathscr{G}}^{(d)}_{k,j_{d},l}\right).
\end{equation}
 based on the only assumption that each elements in TT-cores and the noise precision are independent. To be noticed, $q^{(t)}(\mathbf{\theta})$ is also in a factorized form as $q^{(t-1)}(\mathbf{\theta})$, and can be initialized with $q^{(0)}(\mathbf{\theta}) = p(\mathbf{\theta})$, where $p(\mathbf{\theta})$ is the prior distribution $\mathbf{\theta}$ and equals to the product of \eqref{eq:tt_prior} and \eqref{eq:tau}. From the variational inference framework~\cite{wainwright2008graphical}, $q^{(t)}(\mathbf{\theta})$ is obtained by minimizing the Kullback-Leibler (KL) divergence between $\frac{1}{C}\Tilde{p}(\mathbf{\theta})$ and $q^{(t)}(\mathbf{\theta})$, i.e. $\operatorname{KL}(\frac{1}{C}\Tilde{p}^{(t)}(\mathbf{\theta}||q^{(t)}(\mathbf{\theta}))$, where $C$ is a normalization constant. This is equivalent to maximize a variational model evidence lower bound (ELBO)~\cite{wainwright2008graphical},
\begin{equation}\label{eq:elbo}
    \mathcal{L}(q^{(t)}(\mathbf{\theta})) = \mathbb{E}_{q^{(t)}}\left[\log\left(\frac{\tilde{p}^{(t)}(\mathbf{\theta})}{q^{(t)}(\mathbf{\theta})}\right)\right].
\end{equation}
Due to the factorized form \eqref{eq:mean_field1}, the closed form of the individual factors $q^{(t)}_{\mathbf{\theta}_{i}
}(\mathbf{\theta}_{i})$ can be explicitly derived.
% To be noticed, we initialize the posterior $q_{0}(\mathbf{\theta}) = p(\mathbf{\theta})$.
To be specific, $q^{(t)}_{\mathbf{\theta}_{i}
}(\mathbf{\theta}_{i})$
%The update rule for the $i$-th factor 
based on the maximization of $\mathcal{L}(q^{(t)}(\mathbf{\theta}))$ is then given by
\begin{equation}\label{fac}
    \log q^{(t)}_{\mathbf{\theta}_{i}}(\mathbf{\theta}_{i}) = \mathbb{E}_{q^{(t)}(\theta  \setminus \mathbf{\theta}_{i})}\left[\log \tilde{p}^{(t)}(\mathbf{\theta})\right] + \operatorname{const},
\end{equation}
where $\mathbb{E}_{q^{(t)}(\theta  \setminus \mathbf{\theta}_{i})}$ denotes an expectation w.r.t the $q^{(t)}(\mathbf{\theta})$ distributions over all variables $\theta$ except $\mathbf{\theta}_{i}$. 

\subsubsection{Posterior Distribution of TT-cores}\label{subsec:ptt}
Upon the factorized form (\ref{eq:mean_field1}), the inference of an element of TT-cores $\pmb{\mathscr{G}}_{k,j_{d},l}^{(d)}$ can be performed by receiving the messages from the $t$-th data batch $B_{t}$ and the updated information of the whole TT-cores, where $k=1,\ldots,r_{d-1}$, $l =1,\ldots r_{d}$, $j_{d} = 1,\ldots,N_{d}$. By applying (\ref{fac}), 
the optimized form of $\pmb{\mathscr{G}}_{k,j_{d},l}^{(d)}$ is given as follows,
\begin{equation}\label{eq:uptt}
\log q^{(t)}_{\pmb{\mathscr{G}}_{k,j_{d},l}^{(d)}}(\pmb{\mathscr{G}}_{k,j_{d},l}^{(d)}) = \mathbb{E}_{q^{(t)}(\mathbf{\theta}\setminus \pmb{\mathscr{G}}_{k,j_{d},l}^{(d)})}\left[\log \tilde{p}^{(t)}(\mathbf{\theta})\right]+\operatorname{const}.
\end{equation}
Obviously, the approximation posterior keeps the same form as the updated one, i.e., Gaussian distribution,
\begin{equation}
    \label{eq:ttpos}
    q^{(t)}_{\pmb{\mathscr{G}}_{k,j_{d},l}^{(d)}}(\pmb{\mathscr{G}}_{k,j_{d},l}^{(d)}) = \mathcal{N}(\pmb{\mathscr{G}}_{k,j_{d},l}^{(d)}|\mu_{\pmb{\mathscr{G}}_{k,j_{d},l}^{(d)}}^{(t)},\nu_{\pmb{\mathscr{G}}_{k,j_{d},l}^{(d)}}^{(t)}).
\end{equation}
where $(t)$ represents receiving the $t$-th data batch $B_t$.
Here the variance and mean
% , derived from derivation (\ref{eq:uptt}) w.r.t $\pmb{\mathscr{G}}_{k,j_{d},l}^{(d)}$, 
can be updated by

\begin{equation}\label{core-variance}
\begin{split}
    \nu_{\pmb{\mathscr{G}}_{k,j_{d},l}^{(d)}}^{(t)} &= \Bigg({\nu_{\pmb{\mathscr{G}}_{k,j_{d},l}^{(d)}}^{(t-1)}}^{-1}+\mathbb{E}[\tau]\cdot\sum_{\mathbf{j}\in\Omega_{t}^{j_{d}}}\mathbb{E}\left[b^{(<d)}_{(k-1)r_{d-1}+k}\right]\mathbb{E}\left[b^{(>d)}_{(l-1)r_{d}+l}\right]\Bigg)^{-1},
\end{split}
\end{equation}
\begin{equation}\label{core-mean}
\begin{aligned}
&\mu_{\pmb{\mathscr{G}}_{k,j_{d},l}^{(d)}}^{(t)} = \nu_{\pmb{\mathscr{G}}_{k,j_{d},l}^{(d)}}^{(t)}\mu_{\pmb{\mathscr{G}}_{k,j_{d},l}^{(d)}}^{(t-1)}{\nu_{\pmb{\mathscr{G}}_{k,j_{d},l}^{(d)}}^{(t-1)}}^{-1}+\nu_{\pmb{\mathscr{G}}_{k,j_{d},l}^{(d)}}^{(t)}\mathbb{E}[\tau]\sum_{\mathbf{j}\in\Omega_{t}^{j_{d}}}\Bigg(x_{\mathbf{j}}\mathbb{E}\left[e^{(<d)}_{k}\right]\mathbb{E}\left[e^{(>d)}_{l}\right]\\
&-\sum_{k^{\prime}=1}^{r_{d-1}}\sum_{l^{\prime}=1}^{r_{d}} \mathbb{E}\left[b^{(<d)}_{(k-1)r_{d-1}+k^{\prime}}\right]\mathbb{E}\left[b^{(>d)}_{(l-1)r_{d}+l^{\prime}}\right]\mu_{\pmb{\mathscr{G}}^{(d)}_{k^{\prime},j_{d},l^{\prime}}}^{(t-1)}+ \mathbb{E}\left[b^{(<d)}_{(k-1)r_{d-1}+k}\right]\mathbb{E}\left[b^{(>d)}_{(l-1)r_{d}+l}\right]\mu_{\pmb{\mathscr{G}}^{(d)}_{k,j_{d},l}}^{(t-1)}\Bigg),\\
\end{aligned}
\end{equation}
% \begin{equation}\label{core-mean}
% \begin{aligned}
% \mu_{\pmb{\mathscr{G}}_{k,j_{d},l}^{(d)}}^{(t)} &= \nu_{\pmb{\mathscr{G}}_{k,j_{d},l}^{(d)}}^{(t)}\left(\mu_{\pmb{\mathscr{G}}_{k,j_{d},l}^{(d)}}^{(t-1)}{\nu_{\pmb{\mathscr{G}}_{k,j_{d},l}^{(d)}}^{(t-1)}}^{-1}+\mathbb{E}[\tau]\sum_{\mathbf{j}\in\Omega_{t}^{j_{d}}}\Bigg(x_{\mathbf{j}}\mathbb{E}\left[e^{(<d)}_{k}\right]\mathbb{E}\left[e^{(>d)}_{l}\right]-\right.\\
% &\left.\sum_{\substack{k^{\prime}=1 \\ k^{\prime}\neq k}}^{r_{d-1}}\sum_{\substack{l^{\prime}=1\\ l^{\prime}\neq l}}^{r_{d}} \mathbb{E}\left[b^{(<d)}_{(k-1)r_{d-1}+k^{\prime}}\right]\mathbb{E}\left[b^{(>d)}_{(l-1)r_{d}+l^{\prime}}\right]\mu_{\pmb{\mathscr{G}}^{(d)}_{k^{\prime},j_{d},l^{\prime}}}^{(t-1)}\Bigg)\right),\\
% \end{aligned}
% \end{equation}
with
\begin{equation}
\label{eq:td}
    \mathbb{E}\left[e^{(<d)}\right] = \prod_{i=1}^{d-1} \mathbb{E}\left[\pmb{\mathscr{G}}^{(i)}_{j_{i}}\right],\quad \mathbb{E}\left[e^{(>d)}\right] = \prod_{i=d+1}^{D} \mathbb{E}\left[\pmb{\mathscr{G}}^{(i)}_{j_{i}}\right],
\end{equation}
\begin{equation}
\label{eq:bd}
\begin{split}
    \mathbb{E}\left[b^{(<d)}\right] &= \prod_{i=1}^{d-1}\mathbb{E}\left[\pmb{\mathscr{G}}^{(i)}_{j_{i}}\otimes \pmb{\mathscr{G}}^{(i)}_{j_{i}}\right],\\
    \mathbb{E}\left[b^{(>d)}\right] &= \prod_{i=d+1}^{D}\mathbb{E}\left[\pmb{\mathscr{G}}^{(i)}_{j_{i}}\otimes \pmb{\mathscr{G}}^{(i)}_{j_{i}}\right],
\end{split}
\end{equation}
where $\nu_{\pmb{\mathscr{G}}_{k,j_{d},l}^{(d)}}^{(t)}$ and $\mu_{\pmb{\mathscr{G}}_{k,j_{d},l}^{(d)}}^{(t)}$ are the variance and mean of the element $\pmb{\mathscr{G}}_{k,j_{d},l}^{(d)}$ at the $t$-th data batch. The subscript $\Omega_{t}^{j_{d}}$ represents the indexes of $t$-th batch data with $d$-th index fixing as $j_{d}$. 
The expression $\mathbb{E}[\cdot]$ denotes the expectation with respect to all random variables involved. From (\ref{eq:td}) and (\ref{eq:bd}), we can see that, for $d = 1,\ldots, D$, $e^{(<d)}$ and $b^{(<d)}$ are row vectors, while $e^{(>d)}$ and $b^{(>d)}$ are column vectors. The subscript of them denotes an element of a vector, e.g. $e^{(<d)}_{k}$ is the $k$-th element of vector $e^{(<d)}$. The derivations of \eqref{core-variance} and \eqref{core-mean} are illustrated in Appendix \ref{appendix2}. 

The variance and mean of TT-cores can be initialized by \eqref{eq:tt_prior}, i.e., $q^{(0)}_{\pmb{\mathscr{G}}}(\pmb{\mathscr{G}})= p(\pmb{\mathscr{G}})$.  Moreover, we introduce the following proposition from~\cite{9513808} to compute $\mathbb{E}\left[\pmb{\mathscr{G}}^{(i)}_{j_{i}}\otimes \pmb{\mathscr{G}}^{(i)}_{j_{i}}\right]$ in (\ref{eq:bd}).
\begin{proposition}\label{pro2}
The expectation of Kronecker product $\pmb{\mathscr{G}}^{(i)}_{j_{i}}\otimes \pmb{\mathscr{G}}^{(i)}_{j_{i}}$ in (\ref{eq:bd}) can be calculated by 
% For a random matrices $\mathscr{A}$, if any two elements in $\mathscr{A}$ are independent, then the expectation of Kronecker product $\mathscr{A}\otimes \mathscr{A}$ is 
\begin{equation*}
    \begin{split}
        \mathbb{E}\left[\pmb{\mathscr{G}}^{(i)}_{j_{i}}\otimes \pmb{\mathscr{G}}^{(i)}_{j_{i}}\right] =  &\mathbb{E}\left[\pmb{\mathscr{G}}^{(i)}_{j_{i}}\right] \otimes \mathbb{E}\left[\pmb{\mathscr{G}}^{(i)}_{j_{i}}\right] + \mathbb{E}\left[\left(\pmb{\mathscr{G}}^{(i)}_{j_{i}}-\mathbb{E}\left[\pmb{\mathscr{G}}^{(i)}_{j_{i}}\right]\right)\otimes \left(\pmb{\mathscr{G}}^{(i)}_{j_{i}}-\mathbb{E}\left[\pmb{\mathscr{G}}^{(i)}_{j_{i}}\right]\right)\right],
    \end{split}
\end{equation*}
\end{proposition}
\begin{proof}
From equation (\ref{eq:mean_field1}), we can see any two elements in $\pmb{\mathscr{G}}^{(i)}_{j_{i}}$ are independent,
\begin{equation*}
    \begin{split}
    &\mathbb{E}\left[\pmb{\mathscr{G}}^{(i)}_{k,j_{i},l} \pmb{\mathscr{G}}^{(i)}_{k^{\prime},j_{i},l^{\prime}}\right] =  \mathbb{E}\left[\pmb{\mathscr{G}}^{(i)}_{k,j_{i},l}\right] \mathbb{E}\left[\pmb{\mathscr{G}}^{(i)}_{k^{\prime},j_{i},l^{\prime}}\right] + v_{\pmb{\mathscr{G}}_{k,j_{i}, l }^{(i)}} \delta\left(k-k^{\prime}\right) \delta\left(l-l^{\prime}\right).
    \end{split}
\end{equation*}
where $\delta(x)=1$ if $x=0$ and zero otherwise. 
Let $p =r_{i-1}(k-1)+k^{\prime},\ q = r_{i}(l-1)+l^{\prime}$, then 
% Let $p =r_{i-1}(k-1)+l,\ q = r_{i}(k^{\prime}-1)+l^{\prime}$, then 
\begin{equation*}
    \mathbb{E}\left[\pmb{\mathscr{G}}^{(i)}_{j_{i}}\otimes \pmb{\mathscr{G}}^{(i)}_{j_{i}}\right]_{p,q} =\left(\mathbb{E}\left[\pmb{\mathscr{G}}^{(i)}_{j_{i}}\right]\otimes\mathbb{E}\left[\pmb{\mathscr{G}}^{(i)}_{j_{i}}\right]\right)_{p,q}+\pmb{\mathscr{V}}_{p,j_{i},q}^{(i)},
\end{equation*}
where 
\begin{equation}\label{kroV}
    \pmb{\mathscr{V}}_{j_d}^{(i)}:=\mathbb{E}\left[\left(\pmb{\mathscr{G}}_{j_i}^{(i)}-\mathbb{E}\left[\pmb{\mathscr{G}}_{j_i}^{(i)}\right]\right) \otimes\left(\pmb{\mathscr{G}}_{j_i}^{(i)}-\mathbb{E}\left[\pmb{\mathscr{G}}_{j_i}^{(i)}\right]\right)\right].
\end{equation}
 As a Kronecker-form covariance, $\pmb{\mathscr{V}}_{j_i}^{(i)} \in \mathbb{R}^{r_{i-1}^2 \times r_{i}^2}$ consists of block matrices with dimension $r_{i-1} \times r_{i}$ and the $(k, l)$-th block matrix contains only one nonzero element $v_{\pmb{\mathscr{G}}_{k,j_{i}, l }^{(i)}}$ at the $(k, l)$-th position. Thus, 
 
\begin{equation*}
    \mathbb{E}\left[\pmb{\mathscr{G}}^{(i)}_{j_{i}}\otimes \pmb{\mathscr{G}}^{(i)}_{j_{i}} \right]=\mathbb{E}\left[\pmb{\mathscr{G}}^{(i)}_{j_{i}}\right]\otimes\mathbb{E}\left[\pmb{\mathscr{G}}^{(i)}_{j_{i}}\right]+\pmb{\mathscr{V}}_{j_{i}}^{(i)}.
\end{equation*}

\end{proof}

\subsubsection{Posterior Distribution of Noise Precision}
Following equation (\ref{fac}), the update for the noise precision $\tau$ can be derived by rearranging the equation as follows,
\begin{equation}\label{eq:uptau}
    \log q^{(t)}_{\tau}(\tau) = \mathbb{E}_{q^{(t)}(\mathbf{\theta} \setminus \tau)}\left[\log \tilde{p}^{(t)}(\mathbf{\theta})\right] + \operatorname{const},
\end{equation}
from which, we can see the posterior of noise precision $\tau$ is also a Gamma distribution, 
\begin{equation}
q^{(t)}_{\tau}(\tau) = \operatorname{Gam}(\tau|\alpha^{(t)},\beta^{(t)}),
\end{equation}
and the posterior hyperparameters can be updated by
\begin{equation} \label{tau-alpha}
    \alpha^{(t)} =\alpha^{(t-1)}+ \frac{|\Omega_{t}|}{2},
\end{equation}
\begin{equation} \label{tau-beta}
    \begin{aligned} 
\beta^{(t)} =\beta^{(t-1)}+\frac{1}{2}\mathbb{E}\left[\sum_{\mathbf{j}\in\Omega_{t}}\left(x_{\mathbf{j}}-\prod_{d=1}^{D} \pmb{\mathscr{G}}^{(d)}_{j_{d}}\right)^{2}\right],
    \end{aligned}
\end{equation}
where
\begin{equation}\label{sum}
\begin{split}
    &\mathbb{E}\left[\sum_{\mathbf{j}\in\Omega_{t}}\left(x_{\mathbf{j}}-\prod_{d=1}^{D} \pmb{\mathscr{G}}^{(d)}_{j_{d}}\right)^{2}\right]= \pmb{\mathscr{X}}_{\Omega_{t}}^{T}\left(\pmb{\mathscr{X}}_{\Omega_{t}}-2\pmb{\mathscr{A}}_{\Omega_{t}} \right)+\sum_{\mathbf{j}\in\Omega_{t}} \prod_{d=1}^{D}\left(\mathbb{E}\left[\pmb{\mathscr{G}}^{(d)}_{j_{d}}\right]\otimes\mathbb{E}\left[\pmb{\mathscr{G}}^{(d)}_{j_{d}}\right]+\pmb{\mathscr{V}}_{j_{d}}^{(d)}\right).
\end{split}
\end{equation}
 Appendix \ref{appendix3} provides more details about the update formula. From \eqref{tau-alpha}, we can see that the shape parameter updates with the number of the newly receiving data batch $B_{t}$. The summation term in \eqref{tau-beta} controls the noise precision $\tau$ through the rate parameter $\hat{\beta}$. It is seen in \eqref{tau-beta} that when the model dose not fit the streaming data well, there is an increase in $\hat{\beta}$ and thus a decrease in the noise variance since $\mathbb{E}[\tau]=\hat{\alpha}/\hat{\beta}$. Computing (\ref{sum}) is challenging, so we present a proposition as follows.
\begin{proposition}\label{pro1}
Given a set of independent random matrices $\{\pmb{\mathscr{G}}_{j_{d}}^{(d)}\}_{d=1}^{D}$, we assume that $\forall d\in \{1,\ldots,D\}, \forall k\in\{1\,\ldots,r_{d-1}\}, \forall l\in\{1\,\ldots,r_{d}\}$, the vectors $\{\pmb{\mathscr{G}}_{k,:,l}^{(d)}\}$ are independent, then
\begin{equation*}
\mathbb{E}\left[\left(\prod_{d=1}^{D} \pmb{\mathscr{G}}^{(d)}_{j_{d}}\right)^{2}\right] =\prod_{d=1}^{D}\left(\mathbb{E}\left[\pmb{\mathscr{G}}^{(d)}_{j_{d}}\right]\otimes\mathbb{E}\left[\pmb{\mathscr{G}}^{(d)}_{j_{d}}\right]+\pmb{\mathscr{V}}_{j_{d}}^{(d)}\right).
 \end{equation*}
 \end{proposition}
 \begin{proof}
 For a scalar $\prod_{d=1}^{D} \pmb{\mathscr{G}}^{(d)}_{j_{d}}$, using the mix product property of Kronecker product and conclusion of proposition \ref{pro2},  we have 
 \begin{equation*}
 \begin{aligned}
      \mathbb{E}\left[\left(\prod_{d=1}^{D} \pmb{\mathscr{G}}^{(d)}_{j_{d}}\right)^{2}\right] &= \mathbb{E}\left[\left(\prod_{d=1}^{D} \pmb{\mathscr{G}}^{(d)}_{j_{d}}\right)\otimes \left(\prod_{d=1}^{D} \pmb{\mathscr{G}}^{(d)}_{j_{d}}\right)\right]\\
      &=\prod_{d=1}^{D}\mathbb{E}\left[\left(\pmb{\mathscr{G}}^{(d)}_{j_{d}} \otimes \pmb{\mathscr{G}}^{(d)}_{j_{d}}\right)\right]\\
    &=\prod_{d=1}^{D}\left(\mathbb{E}\left[\pmb{\mathscr{G}}^{(d)}_{j_{d}}\right]\otimes\mathbb{E}\left[\pmb{\mathscr{G}}^{(d)}_{j_{d}}\right]+\pmb{\mathscr{V}}_{j_{d}}^{(d)}\right),
 \end{aligned}
 \end{equation*}
where $\pmb{\mathscr{V}}_{j_d}^{(d)}$ is defined in (\ref{kroV}). 
\end{proof}

\subsubsection{Algorithm}
Since the ELBO in \eqref{eq:elbo} is only guaranteed to converge to a local maximum\cite{blei2017variational}, a good initialization is important. The noise precision $\tau$ in \eqref{eq:tau} is initialized as  $q^{(0)}_{\tau}(\tau) = \operatorname{Gam}(\tau|\alpha_{0},\beta_{0})$ where $\alpha_{0}=\beta_{0}=10^{-3}$, and then $\mathbb{E}_{q^{(0)}_{\tau}}[\tau] = 1$. For TT-cores $\{\pmb{\mathscr{G}}^{(d)}\}_{d=1}^{D}$ in \eqref{eq:tt_prior}, each fiber $\pmb{\mathscr{G}}^{(d)}_{k,:,l}$ is initialized as, 
\begin{equation}\label{ini_G}
    q^{(0)}_{\pmb{\mathscr{G}}^{(d)}_{k,:,l}}(\pmb{\mathscr{G}}^{(d)}_{k,:,l}) = \mathcal{N}(\pmb{\mathscr{G}}^{(d)}_{k,:,l}|\mathbf{m}_{\pmb{\mathscr{G}}^{(d)}_{k,:,l}},\mathbf{I}),
\end{equation}
where each element of $\mathbf{m}_{\pmb{\mathscr{G}}^{(d)}_{k,:,l}}$ is generated from a standard uniform distribution $U(0,1)$. Then  $\mathbb{E}_{q^{(0)}_{\pmb{\mathscr{G}}^{(d)}}}(\pmb{\mathscr{G}}^{(d)}) = \mathbf{m}_{\pmb{\mathscr{G}}^{(d)}}$, where $\mathbf{m}_{\pmb{\mathscr{G}}^{(d)}}$ is a tensor with the same size as $\pmb{\mathscr{G}}^{(d)}$ and collects the fibers $\pmb{\mathscr{G}}^{(d)}_{k,:,l}$, for $k=1,\ldots,r_{d}$, $l = 1,\ldots,r_{d-1}$. 

% The stopping criterion for updating the mean and variance of TT-cores based on the data batch $B_{t}$ in \eqref{core-variance}--\eqref{core-mean} is defined as 
% % The following relative error of TT-cores between two iterations is defined as stopping criterion for data batch $B_{t}$,
% \begin{equation}\label{epG}
%     \epsilon_{\mathscr{G}} = \sum_{d=1}^{D}\frac{\Big\|\mathbb{E}_{q^{(t,m)}_{\pmb{\mathscr{G}}^{(d)}}}\left[\pmb{\mathscr{G}}^{(d)} \right]-\mathbb{E}_{q^{(t,m-1)}_{\pmb{\mathscr{G}}^{(d)}}}\left[ \pmb{\mathscr{G}}^{(d)} \right]\Big\|_{F}}{
%     \Big\|\mathbb{E}_{q^{(t,m-1)}_{\pmb{\mathscr{G}}^{(d)}}}\left[ \pmb{\mathscr{G}}^{(d)} \right]\Big\|_{F}},
% \end{equation}
% where $\|\cdot\|_{F}$ is the Frobenius norm generalized to tensor. $\mathbb{E}_{q^{(t,m)}_{\pmb{\mathscr{G}}^{(d)}}}[\cdot]$ denotes the expectation of the parameter after the $m$-th iteration using the $t$-th data batch. For each data batch $B_{t}$, we declare our update scheme converged when the relative error $\epsilon_{\mathscr{G}}$  is less than a tolerance $\operatorname{tol} = 10^{-3}$. 
Our SPTT algorithm is summarized in Algorithm~\ref{alg:sptt}, where we stop the algorithm and output the updated mean and variance of each TT-cores if the final data batch $B_{t_{\max}}$ is used.
\begin{algorithm}
	\caption{Streaming probabilistic tensor train decomposition (SPTT)}
       \label{alg:sptt}
	\begin{algorithmic}[1]
		\Require Data streams $\{B_{1},\ldots,B_{t_{\max}}\}$.
  % maximum iteration number $M$ for each data batch.
  \State Initialize $q^{(0)}_{\tau}(\tau)$ using \eqref{eq:tau}, $\alpha_{0}=\beta_{0}=10^{-3}$, then $\mathbb{E}_{q^{(0)}_{\tau}}[\tau] = 1$. 
\State Initialize $q^{(0)}_{\pmb{\mathscr{G}}^{(d)}}(\pmb{\mathscr{G}}^{(d)})$ using \eqref{ini_G}, for $d=1,\ldots,D$.
\State Set $t=1$ and $\epsilon_{\mathscr{G}} = 1$.
\While{$t<t_{\max}$}
\State Load the data batch $B_{t}$.
% \State Set $m=0$, $q^{(0,0)}_{\pmb{\mathscr{G}}^{(d)}}$.
\While{Not converge}
% \While{$\epsilon_{\mathscr{G}}> 10^{-3}$ and  iteration number $m<M$}
\For{$d = 1$ to $D$}
% \State{Calculate the expectations of $e^{(<d)}$ ,$b^{(<d)}$, $e^{(>d)}$, $b^{(>d)}$ by (\ref{eq:td}) and (\ref{eq:bd}).}
\State{Calculate $\mathbb{E}\left[e^{(<d)}\right]$,$\mathbb{E}\left[b^{(<d)}\right]$, $\mathbb{E}\left[e^{(>d)}\right]$, $\mathbb{E}\left[b^{(>d)}\right]$ by (\ref{eq:td}) and (\ref{eq:bd}).}
% $\mathbb{E}_{q^{(t,m)}}\left[e^{(<d)^{T}}\right]$, $\mathbb{E}_{q^{(t,m)}}\left[b^{(<d)^{T}}\right]$,  $\mathbb{E}_{q^{(t,m)}}\left[e^{(>d)}\right]$ and $\mathbb{E}_{q^{(t,m)}}\left[b^{(>d)}\right]$ by (\ref{eq:td}) and (\ref{eq:bd}).}
\State {
Update $q^{(t)}_{\pmb{\mathscr{G}}^{(d)}}$ associated with $B_{t}$ via (\ref{core-variance}) and (\ref{core-mean}).}
\State{Update $q^{(t)}_{\tau}$ via (\ref{tau-alpha}) and (\ref{tau-beta}).}
% \State Compute the error $\epsilon_{\mathscr{G}}$ in \eqref{epG}.
\EndFor
% \State Let $m = m+1$.
\EndWhile
% \State Let $q^{(t+1,0)}_{\pmb{\mathscr{G}}^{(d)}}$=$q^{(t,m)}_{\pmb{\mathscr{G}}^{(d)}}$, $q^{(t+1,0)}_{\tau}$=$q^{(t,m)}_{\tau}$.
\State Let $t = t+1$.
\EndWhile
\State Let $q^{*}_{\pmb{\mathscr{G}}^{(d)}}=q^{(t_{\max})}_{\pmb{\mathscr{G}}^{(d)}},\  d=1,\ldots D$.
% \State Let $q^{*}_{\pmb{\mathscr{G}}^{(d)}}=q^{(t_{\max},m)}_{\pmb{\mathscr{G}}^{(d)}},\  d=1,\ldots D$, where $m$ represents the last iteration.
\Ensure The updated posterior of TT-cores $q^{*}_{\pmb{\mathscr{G}}^{(d)}},\  d=1,\ldots, D$.
	\end{algorithmic}
\end{algorithm}

The complexity of the proposed algorithm arises from updating the posterior of $\{\pmb{\mathscr{G}}^{(d)}\}_{d=1}^{D}$ and $\tau$. For simplicity, we assume that all TT-ranks are initially set as $L$, the length in each order is $N$, and the batch size is $S$. In each streaming batch, updating of each TT core $\pmb{\mathscr{G}}^{(d)}$ involves $\mathbf{O}(DL^{2})$ operations to get $e^{(<d)}$ and $e^{(>d)}$, $\mathbf{O}(DL^{4})$ operations to obtain $b^{(<d)}$ and $b^{(>d)}$, $\mathbf{O}(SDL^{2})$ operations to compute the mean and $\mathbf{O}(SDL^{4})$ operations to calculate the variance of a TT core in (\ref{core-variance}) and (\ref{core-mean}) respectively. Furthermore, computing each $\tau$ requires $\mathbf{O}(SD(L^4 + L^2 ))$ operations. Thus, the overall time complexity is $\mathbf{O}(SDL^4)$. Thus, the overall time complexity is $\mathbf{O}(SDL^4)$. The space complexity is $\mathbf{O}(NDL^{2})$, which is required to store the posterior mean and variance of $\{\pmb{\mathscr{G}}^{(d)}\}_{d=1}^{D}$.   
% \subsection{Derivative of the update process}
% In this section we provide the main steps of deriving our TT-cores and the noise term update. In 
\section{Experiments}
\label{sec:experiments}
Based on both high-order synthetic and real-world applications, our SPTT are compared with related decomposition methods, including Bayesian streaming tensor decomposition methods: POST~\cite{du2018probabilistic},  BASS-Tucker~\cite{fang2021bayesian}, and static tensor decomposition methods: CP-ALS~\cite{tensortoolbox}, CP-WOPT~\cite{acar2011scalable}, Tucker-ALS~\cite{de2000best}. All results of this paper are obtained in MATLAB (2022a) on a computer with an Intel(R) Core(TM) i7-8650U CPU @ 1.90GHz 2.11 GHz. 
% The codes are available at \href{https://anonymous.4open.science/r/SPTT-E6DE}{https://anonymous.4open.science/r/SPTT-E6DE}

\subsection{Synthetic Data}
The synthetic tensor data is generated by the following procedure. Through TT decomposition in \eqref{TTdecom}, we create a true tensor $\pmb{\mathscr{A}}\in \mathbb{R}^{20\times 20\times 20\times 20}$ in TT-format with the TT-ranks $(1,3,3,3,1)$ by sampling the elements of TT-cores $\{\pmb{\mathscr{G}}^{(d)}\}_{d=1}^{4}$ from standard uniform distribution $U(0,1)$, and $\pmb{\mathscr{A}} = \langle\langle \pmb{\mathscr{G}}^{(1)},\pmb{\mathscr{G}}^{(2)},\pmb{\mathscr{G}}^{(3)},\pmb{\mathscr{G}}^{(4)}\rangle\rangle$. The synthetic observed tensor $\pmb{\mathscr{U}}$ is generated as follows,
    \begin{equation}
        \pmb{\mathscr{U}} = \pmb{\mathscr{A}}+\pmb{\mathscr{W}},
    \end{equation}
    where $\pmb{\mathscr{W}}$ is a Gaussian noise tensor with its element $w_{\mathbf{j}}\sim \mathcal{N}(0,\sigma^{2})$ and $\sigma$ is noise level, controlled by signal-to-noise-ratio (SNR),
    \begin{equation}
        \operatorname{SNR} = 10 \log\left(\frac{\operatorname{var}(\pmb{\mathscr{A}})}{\sigma^{2}}\right).
    \end{equation}
    %     \begin{equation}
    %     \operatorname{SNR} = 10\log\left(\frac{\Vert \pmb{\mathscr{A}} \Vert_{F}^{2}}{ \Vert \pmb{W}\Vert_{F}^{2}}\right) = 10 \log\left(\frac{\operatorname{var}(\pmb{\mathscr{A}})}{\sigma^{2}}\right).
    % \end{equation}
In order to investigate the ability of our streaming decomposition algorithm, the ratio of the number of observed tensor elements to the total tensor elements is set to $15\%$, and the indexes of observed elements are chosen randomly. We divide $10\%$ observed elements into the test set $T$, remaining $90\%$ of observed elements into the training set $B$. Then the training set $B$ is randomly partitioned into a stream of small data batches $\{B_{1},B_{2},\ldots, B_{t_{\max}}\}$. The batch size is denoted by $S$, then $t_{\max} = |B|/S$ in Algorithm \ref{alg:sptt}, where $|B|$ is the total number of the training set. Moreover, the entire observed data is input to the static decomposition algorithms (CP-ALS, CP-WOPT, and Tuckers-ALS). The maximum iteration number is set to $M = 100$ in Algorithm \ref{alg:sptt}, and TT-ranks of initial TT-cores are denoted by a vector $(1, R, R, R,1)$. In related tensor decomposition methods experiments, default settings are used for maximum iteration number and initialization.

    Denote $\hat{\pmb{\mathscr{A}}}$ as the TT-format tensor formulated by the updated posterior mean of TT-cores. To evaluate the performance of the above tensor decomposition approaches, the following relative error $\epsilon$ is used as reconstruction criterion,
    \begin{equation}\label{rerror}
        \epsilon= \frac{\Vert \hat{\pmb{\mathscr{A}}}_{T} - \pmb{\mathscr{A}}_{T} \Vert_{F}}{\|\pmb{\mathscr{A}}_{T}\|_{F}}.
    \end{equation}
    where $\hat{\pmb{\mathscr{A}}}_{T}$ denote a vector contains elements of tensor $\hat{\pmb{\mathscr{A}}}$ on test set $T$. 
    %So as $\pmb{\mathscr{A}}_{T}$. 

 Next, our SPTT is compared with related tensor decomposition methods (POST, BASS-Tucker, CP-ALS, CP-WOPT and Tucker-ALS) regarding the relative reconstruction error under different batch sizes $S$, different initial rank $R$, and different SNR.
    To be noticed, the rank $R$ in CP, Tucker and TT decomposition denotes the CP-rank $R$, Tucker-ranks $(R,R,R,R)$ and TT-ranks $(1,R,R,R,1)$ respectively. All tensor decomposition algorithms are repeated $5$ times, and the calculated average relative error and standard deviation are shown in Fig. \ref{fig:syn1}. In Fig. \ref{fig:syn1} (a), the settings are the fixed rank $R = 3$, $\operatorname{SNR} = 20$, and different choices of batch sizes $S \in \{ 2^8, 2^9 , 2^{10}, 2^{11}\}$. In Fig. \ref{fig:syn1} (b), the settings are the fixed streaming batch size $S =512$, $\operatorname{SNR} = 20$, and different rank $R \in \{3,4,5,6\}$. In Fig. \ref{fig:syn1} (c), the settings are $S = 512$, $R = 3$, and different noise level $\operatorname{SNR} \in \{15,20,25,30\}$. As we can see, SPTT outperforms BASS-Tucker, POST and static decomposition methods in all the cases. 

\begin{figure}[ht]
\vskip 0.2in
\begin{center}
\centerline{\includegraphics[width=\columnwidth]{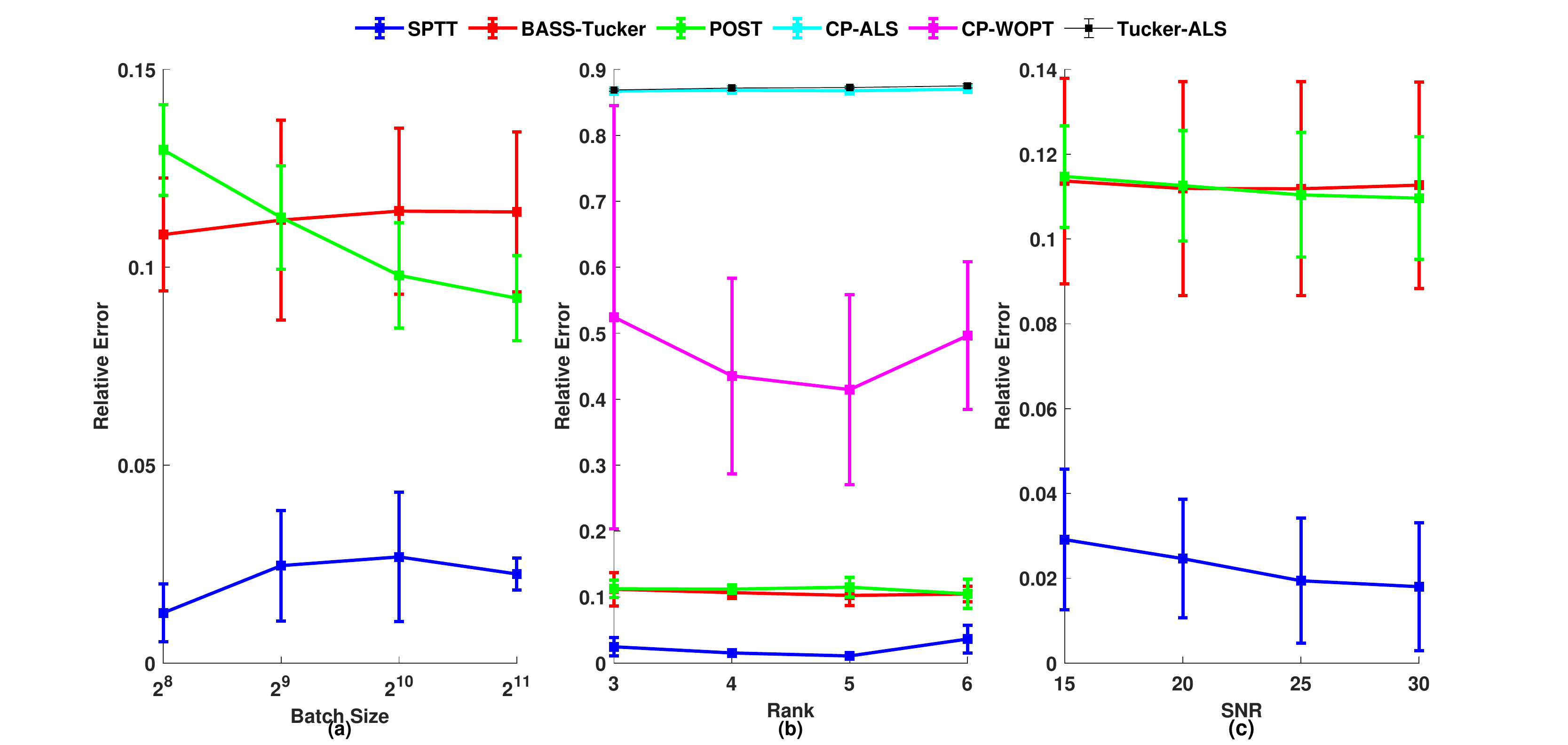}}
\caption{Predictive performance of synthetic data under different conditions.}
\label{fig:syn1}
\end{center}
\vskip -0.2in
\end{figure}

In addition, the running predictive performance of SPTT is evaluated. We fix the batch size $S = 512$, $\operatorname{SNR}=20$, and gradually feed the training set $\{B_{1},B_{2},\ldots, B_{t_{\max}}\}$ to SPTT, BASS-Tucker and POST. Two cases ($R = 3,\,5$) are considered to test the prediction accuracy after each data batch is used. The corresponding running relative error is shown in Fig. \ref{fig:syn2}. As we can see, the relative error of our SPTT is clearly smaller than that of  POST and BASS-Tucker, which demonstrates the accuracy of our SPTT.

\begin{figure}[ht]
\vskip 0.2in
\begin{center}
\centerline{\includegraphics[width=\columnwidth]{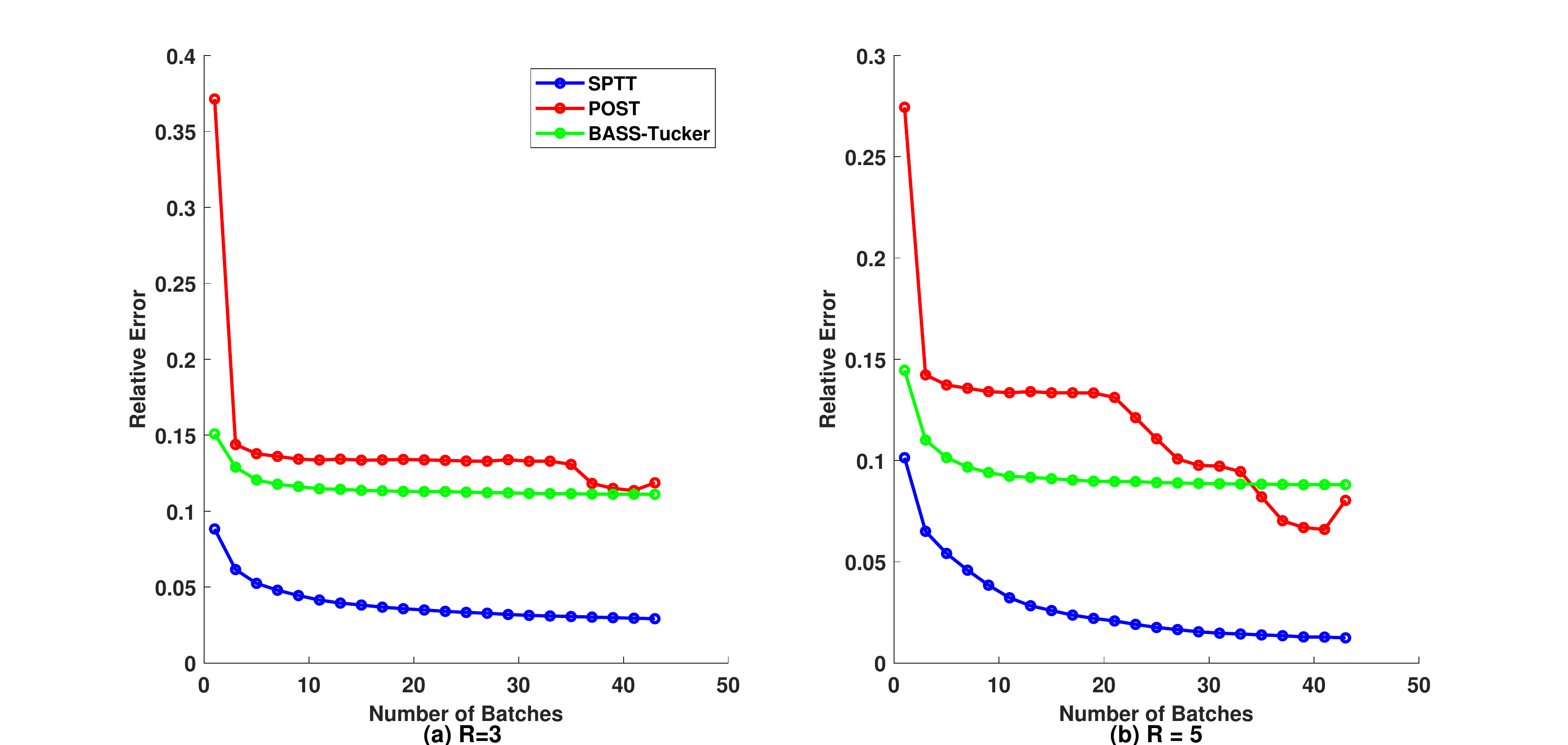}}
\caption{The running predictive performance of synthetic data.}
\label{fig:syn2}
\end{center}
\vskip -0.2in
\end{figure}

\subsection{Real-world Applications}
To evaluate our SPTT, three cases of datasets (ALOG~\cite{zhe2016distributed}, Carphone~\cite{song2020robust}, Flow injection\footnote{Retrieved from \href{www.models.kvl.dk}{www.models.kvl.dk}}) are considered. ALOG data come from a three-mode (user, action, resource) tensor of size $200\times 100\times 200$ including $0.66\%$ observed elements, which represents three-way management operations; Carphone data are extracted from a three-mode (length, width, frames) of video data of size $144\times 176 \times 180$, which contains $0.5\%$ nonzeros elements; Flow injection data,  as a three-mode (samples, wavelengths, times) tensor of size $12\times 100\times 89$, are extracted from a flow injection analysis system and include $21360$ nonzero elements. 

Here our SPTT is compared with the related methods (POST, BASS-Tucker, CP-ALS, CP-WOPT, and Tucker-ALS) in terms of the relative reconstruction error $\epsilon$ defined in \eqref{rerror} under different batch size $S$ and different initial rank $R$. To be noticed, the rank $R$ in CP, Tucker and TT decomposition denotes the CP-rank $R$, Tucker-ranks $(R,R,R)$ and TT-ranks $(1,R,R,1)$ respectively.
The maximum iteration number is set to
$M=100$ in Algorithm \ref{alg:sptt} and TT-ranks of initial TT-cores are set to $(1,R,R,1)$. For maximum iteration number and initialization of all the related tensor decomposition methods, we use their corresponding default settings.   The observed elements of Alog are randomly divided into $75\%$ for training and $25\%$ for testing, and for the other two datasets, their $10\%$ observed elements are prepared for testing. For POST, SPTT, and BASS-Tucker, the training elements are  randomly partitioned into a stream of small batches according to batch size $S$. Moreover, we input the whole observed data for static decomposition algorithms (CP-ALS, CP-WOPT, and Tucker-ALS). Fig. \ref{fig:real1} reports the average relative error and standard deviation. In Fig. \ref{fig:real1} (a)-(c), we fix the batch size $S =512$ and show how the predictive performance of each method varies with different rank $R \in \{3,4,5,6\}$. The bigger the rank, the more expensive for SPTT to factorize each batch. It can be seen from Fig. \ref{fig:real1} (c), as the initial rank $R$ increases from $3$ to $5$, the relative error of SPTT gradually increases and finally gets close to the relative error of POST. 
%Hence this setting shows the trade-off between accuracy and computational complexity. 
To test the capacity of decomposition algorithms on different batch sizes,  we fix the rank $R = 3$, and test the performance with different choices of batch sizes $S \in \{ 2^8, 2^9, 2^{10}, 2^{11}\}$ in Fig. \ref{fig:real1} (d)-(f). As we can see, SPTT outperforms BASS-Tucker, POST, and static decomposition methods for three datasets, especially for Carphone data.

\begin{figure}[ht]
\vskip 0.2in
\begin{center}
\centerline{\includegraphics[width=\columnwidth]{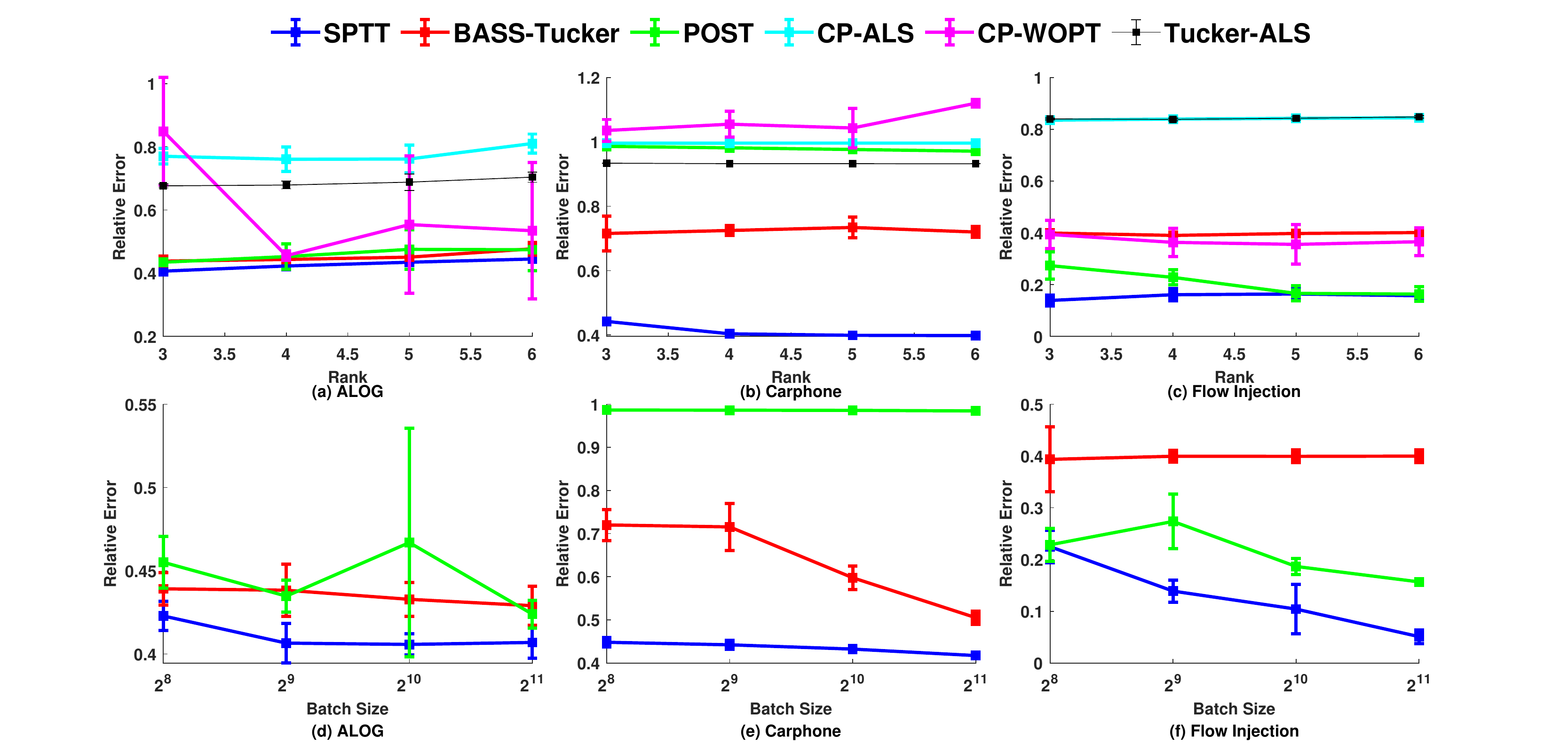}}
\caption{Predictive performance with different rank(top row) and streaming batch size (bottom row).}
\label{fig:real1}
\end{center}
\vskip -0.2in
\end{figure}

The running predictive performance of SPTT is evaluated. We fix the batch size $S = 512$, and continuously feed the training set to SPTT, BASS-Tucker and POST. For the three datasets, two cases are considered, i.e. $R = 3$ and $R = 5$ to test the prediction accuracy after each data batch is used. The running relative error is shown in Fig. \ref{fig:syn2}. In all the cases, when the number of batches is small, POST or BASS-Tucker may performs better than SPTT. However, as the data streams, SPTT beats POST and BASS-Tucker gradually. 
% \begin{figure}
%   \centering
%   \includegraphics[scale = 0.27]{real2-eps-converted-to.pdf}
%   \caption{The running predictive performance in real-world applications.}
%   \label{fig:real2}
% \end{figure}

\begin{figure}[ht]
\vskip 0.2in
\begin{center}
\centerline{\includegraphics[width=\columnwidth]{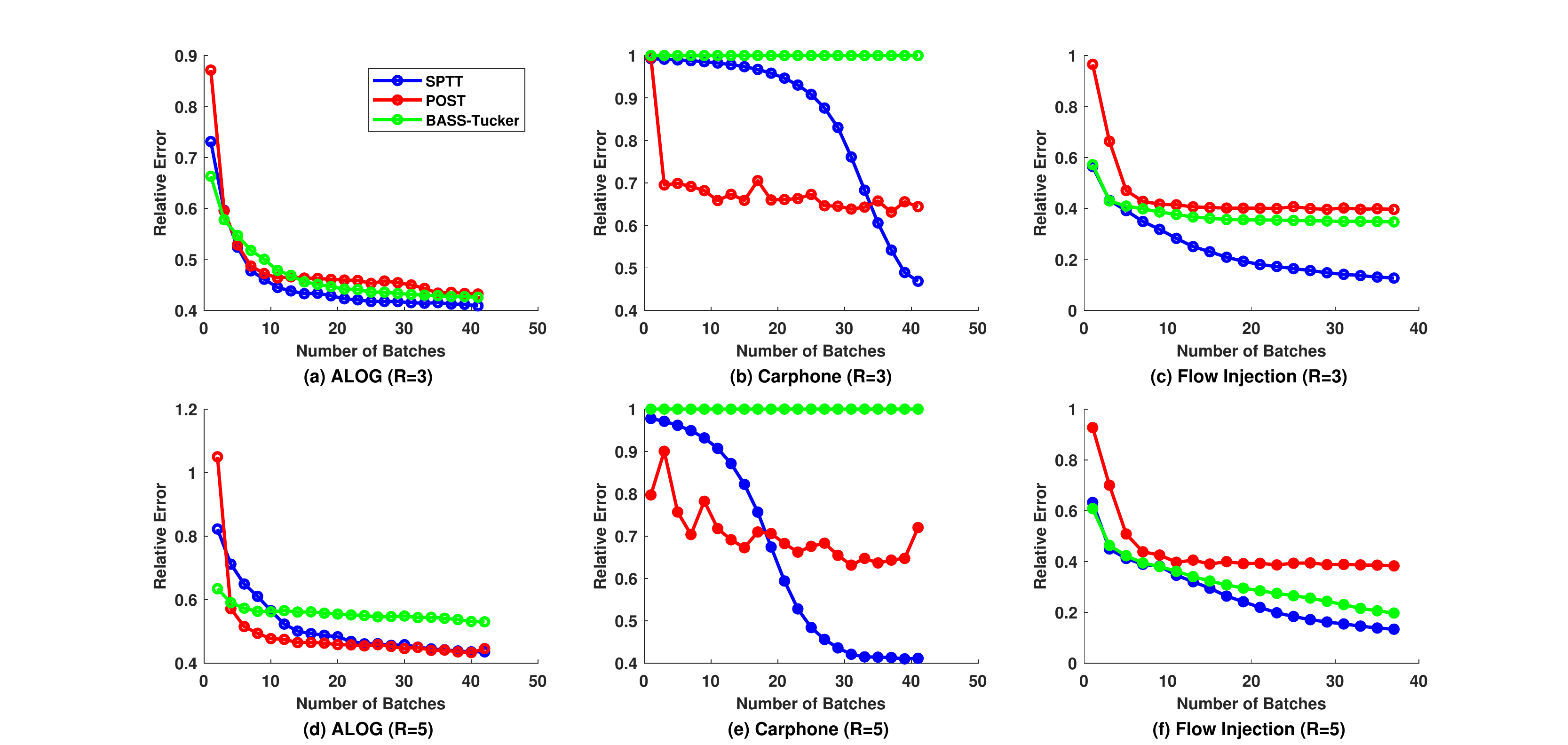}}
\caption{The running predictive performance in real-world applications.}
\label{fig:real2}
\end{center}
\vskip -0.2in
\end{figure}

\section{CONCLUSIONS}
\label{sec:conclusions}
In this paper, we propose a novel Bayesian streaming tensor decomposition algorithm, namely SPTT, to find the latent structure approximation from high-order incomplete and noisy streaming data and predict missing values. Leveraging TT decomposition and a customized Gaussian prior, SPTT addresses key challenges in high-order streaming data recovery. Numerical experiments show that our SPTT algorithm can produce accurate predictions on synthetic and real-world streaming data. For future work, we are interested in develop a sparse strategy based on SPTT to further prevent overfitting and enhance scalability.

% \section*{Acknowledgments}
% This was was supported in part by......

%Bibliography
\bibliographystyle{unsrt}  
\bibliography{references}  

\appendix
% \onecolumn
\section{The variational posterior distribution of TT-cores at time $t$.}\label{appendix2}
\begin{equation}
\begin{aligned}
    &\log q^{(t)}_{\pmb{\mathscr{G}}_{k,j_{d},l}^{(d)}}(\pmb{\mathscr{G}}_{k,j_{d},l}^{(d)}) = \mathbb{E}_{q^{(t)}(\mathbf{\theta}\setminus \pmb{\mathscr{G}}_{k,j_{d},l}^{(d)})}\left[\log \tilde{p}^{(t)}(\mathbf{\theta})\right]+\operatorname{const}\\
    &= \mathbb{E}\left[-\frac{1}{2}{\nu_{\pmb{\mathscr{G}}_{k,j_{d},l}^{(d)}}^{(t-1)}}^{-1}\left(\pmb{\mathscr{G}}^{(d)}_{k,j_{d},l}-\mu_{\pmb{\mathscr{G}}_{k,j_{d},l}^{(d)}}^{(t-1)}\right)^{2}-\frac{\tau}{2}\sum_{\mathbf{j}\in\Omega_{t}^{j_{d}}}\left(x_{\mathbf{j}}-\prod_{d=1}^{D} \pmb{\mathscr{G}}^{(d)}_{j_{d}}\right)^{2}\right]+\operatorname{const}\\
    &=\mathbb{E}\left[-\frac{1}{2}{\nu_{\pmb{\mathscr{G}}_{k,j_{d},l}^{(d)}}^{(t-1)}}^{-1}{\pmb{\mathscr{G}}^{(d)}_{k,j_{d},l}}^{2}+\mu_{\pmb{\mathscr{G}}_{k,j_{d},l}^{(d)}}^{(t-1)}{\nu_{\pmb{\mathscr{G}}_{k,j_{d},l}^{(d)}}^{(t-1)}}^{-1}\pmb{\mathscr{G}}^{(d)}_{k,j_{d},l}-\frac{\tau}{2}\sum_{\mathbf{j}\in\Omega_{t}^{j_{d}}}\left(\prod_{d=1}^{D} \pmb{\mathscr{G}}^{(d)}_{j_{d}}\right)^{2}+\tau\sum_{\mathbf{j}\in\Omega_{t}^{j_{d}}}x_{\mathbf{j}}\prod_{d=1}^{D} \pmb{\mathscr{G}}^{(d)}_{j_{d}}\right]+\operatorname{const}\\
    &=\mathbb{E}\left[ -\frac{1}{2}{\nu_{\pmb{\mathscr{G}}_{k,j_{d},l}^{(d)}}^{(t-1)}}^{-1}{\pmb{\mathscr{G}}^{(d)}_{k,j_{d},l}}^{2}+\mu_{\pmb{\mathscr{G}}_{k,j_{d},l}^{(d)}}^{(t-1)}{\nu_{\pmb{\mathscr{G}}_{k,j_{d},l}^{(d)}}^{(t-1)}}^{-1}\pmb{\mathscr{G}}^{(d)}_{k,j_{d},l}-\frac{\tau}{2}\sum_{\mathbf{j}\in\Omega_{t}^{j_{d}}}\prod_{d=1}^{D}\left(\pmb{\mathscr{G}}^{(d)}_{j_{d}}\otimes \pmb{\mathscr{G}}^{(d)}_{j_{d}}\right)+\tau\sum_{\mathbf{j}\in\Omega_{t}^{j_{d}}}x_{\mathbf{j}}\prod_{d=1}^{D} \pmb{\mathscr{G}}^{(d)}_{j_{d}}\right]+\operatorname{const}\\
    &=\mathbb{E}\left[ -\frac{1}{2}{\nu_{\pmb{\mathscr{G}}_{k,j_{d},l}^{(d)}}^{(t-1)}}^{-1}{\pmb{\mathscr{G}}^{(d)}_{k,j_{d},l}}^{2}+\mu_{\pmb{\mathscr{G}}_{k,j_{d},l}^{(d)}}^{(t-1)}{\nu_{\pmb{\mathscr{G}}_{k,j_{d},l}^{(d)}}^{(t-1)}}^{-1}\pmb{\mathscr{G}}^{(d)}_{k,j_{d},l}-\frac{\tau}{2}\sum_{\mathbf{j}\in\Omega_{t}^{j_{d}}} b^{(<d)}\left(\pmb{\mathscr{G}}^{(d)}_{j_{d}}\otimes \pmb{\mathscr{G}}^{(d)}_{j_{d}}\right)b^{(>d)}+\tau \sum_{\mathbf{j}\in\Omega_{t}^{j_{d}}}x_{\mathbf{j}}e^{(<d)}\pmb{\mathscr{G}}^{(d)}_{j_{d}}e^{(>d)}\right]+\operatorname{const}\\
    &=\mathbb{E}\left[ -\frac{1}{2}{\nu_{\pmb{\mathscr{G}}_{k,j_{d},l}^{(d)}}^{(t-1)}}^{-1}{\pmb{\mathscr{G}}^{(d)}_{k,j_{d},l}}^{2}+\mu_{\pmb{\mathscr{G}}_{k,j_{d},l}^{(d)}}^{(t-1)}{\nu_{\pmb{\mathscr{G}}_{k,j_{d},l}^{(d)}}^{(t-1)}}^{-1}\pmb{\mathscr{G}}^{(d)}_{k,j_{d},l}-\frac{\tau}{2}{\pmb{\mathscr{G}}^{(d)}_{k,j_{d},l}}^{2}\sum_{\mathbf{j}\in\Omega_{t}^{j_{d}}}b^{(<d)}_{(k-1)r_{d-1}+k}b^{(>d)}_{(l-1)r_{d}+l}\right.\\
    & \left.-\tau\pmb{\mathscr{G}}^{(d)}_{k,j_{d},l} \sum_{\mathbf{j}\in\Omega_{t}^{j_{d}}}\left(\sum_{k^{\prime}=1}^{r_{d-1}}\sum_{l^{\prime}=1}^{r_{d}} b^{(<d)}_{(k-1)r_{d-1}+k^{\prime}}b^{(>d)}_{(l-1)r_{d}+l^{\prime}}\pmb{\mathscr{G}}^{(d)}_{k^{\prime},j_{d},l^{\prime}}- b^{(<d)}_{(k-1)r_{d-1}+k}b^{(>d)}_{(l-1)r_{d}+l}\pmb{\mathscr{G}}^{(d)}_{k,j_{d},l}\right)+\tau \pmb{\mathscr{G}}^{(d)}_{k,j_{d},l}\sum_{\mathbf{j}\in\Omega_{t}^{j_{d}}}x_{\mathbf{j}}e^{(<d)}_{k} e^{(>d)}_{l}\right]+\operatorname{const}\\
    &=\mathbb{E}\left[-\frac{1}{2}\left({\nu_{\pmb{\mathscr{G}}_{k,j_{d},l}^{(d)}}^{(t-1)}}^{-1}+\tau \sum_{\mathbf{j}\in\Omega_{t}^{j_{d}}}b^{(<d)}_{(k-1)r_{d-1}+k}b^{(>d)}_{(l-1)r_{d}+l}\right){\pmb{\mathscr{G}}^{(d)}_{k,j_{d},l}}^{2}+\Bigg(\mu_{\pmb{\mathscr{G}}_{k,j_{d},l}^{(d)}}^{(t-1)}{\nu_{\pmb{\mathscr{G}}_{k,j_{d},l}^{(d)}}^{(t-1)}}^{-1}\right.\\&
    \left.+\tau\sum_{\mathbf{j}\in\Omega_{t}^{j_{d}}}\left(x_{\mathbf{j}}e^{(<d)}_{k} e^{(>d)}_{l}-\sum_{k^{\prime}=1}^{r_{d-1}}\sum_{l^{\prime}=1}^{r_{d}} b^{(<d)}_{(k-1)r_{d-1}+k^{\prime}}b^{(>d)}_{(l-1)r_{d}+l^{\prime}}\pmb{\mathscr{G}}^{(d)}_{k^{\prime},j_{d},l^{\prime}}+ b^{(<d)}_{(k-1)r_{d-1}+k}b^{(>d)}_{(l-1)r_{d}+l}\pmb{\mathscr{G}}^{(d)}_{k,j_{d},l}\right)\Bigg)\pmb{\mathscr{G}}_{k,j_{d},l}^{(d)}\right] +\operatorname{const}\\
    &=-\frac{1}{2}\left({\nu_{\pmb{\mathscr{G}}_{k,j_{d},l}^{(d)}}^{(t-1)}}^{-1}+\mathbb{E}\left[\tau\right] \sum_{\mathbf{j}\in\Omega_{t}^{j_{d}}}\mathbb{E}\left[b^{(<d)}_{(k-1)r_{d-1}+k}\right]\mathbb{E}\left[b^{(>d)}_{(l-1)r_{d}+l}\right]\right){\pmb{\mathscr{G}}^{(d)}_{k,j_{d},l}}^{2}\\
    &+\left(\mu_{\pmb{\mathscr{G}}_{k,j_{d},l}^{(d)}}^{(t-1)}{\nu_{\pmb{\mathscr{G}}_{k,j_{d},l}^{(d)}}^{(t-1)}}^{-1}+\mathbb{E}\left[\tau\right]\sum_{\mathbf{j}\in\Omega_{t}^{j_{d}}}\Bigg(x_{\mathbf{j}}\mathbb{E}\left[e^{(<d)}_{k}\right] \mathbb{E}\left[e^{(>d)}_{l}\right]\right.\\
    &\left.-\sum_{k^{\prime}=1}^{r_{d-1}}\sum_{l^{\prime}=1}^{r_{d}} \mathbb{E}\left[b^{(<d)}_{(k-1)r_{d-1}+k^{\prime}}\right]\mathbb{E}\left[b^{(>d)}_{(l-1)r_{d}+l^{\prime}}\right]\mu_{\pmb{\mathscr{G}}^{(d)}_{k^{\prime},j_{d},l^{\prime}}}^{(t-1)}+ 
  \mathbb{E}\left[b^{(<d)}_{(k-1)r_{d-1}+k}\right]\mathbb{E}\left[b^{(>d)}_{(l-1)r_{d}+l}\right]\mu_{\pmb{\mathscr{G}}^{(d)}_{k,j_{d},l}}^{(t-1)}
  \Bigg)\right)\pmb{\mathscr{G}}_{k,j_{d},l}^{(d)}+\operatorname{const}\\
    \end{aligned}
\end{equation}
Thus, $q^{(t)}_{\pmb{\mathscr{G}}_{k,j_{d},l}^{(d)}}(\pmb{\mathscr{G}}_{k,j_{d},l}^{(d)})$ is also a Gaussian distribution
\begin{equation}
    q^{(t)}_{\pmb{\mathscr{G}}_{k,j_{d},l}^{(d)}}(\pmb{\mathscr{G}}_{k,j_{d},l}^{(d)}) = \mathcal{N}(\pmb{\mathscr{G}}_{k,j_{d},l}^{(d)}|\mu_{\pmb{\mathscr{G}}_{k,j_{d},l}^{(d)}}^{(t)},\nu_{\pmb{\mathscr{G}}_{k,j_{d},l}^{(d)}}^{(t)}),
\end{equation}
where
\begin{equation}
\begin{aligned}  \nu_{\pmb{\mathscr{G}}_{k,j_{d},l}^{(d)}}^{(t)} &= \left({\nu_{\pmb{\mathscr{G}}_{k,j_{d},l}^{(d)}}^{(t-1)}}^{-1}+\mathbb{E}\left[\tau\right] \sum_{\mathbf{j}\in\Omega_{t}^{j_{d}}}\mathbb{E}\left[b^{(<d)}_{(k-1)r_{d-1}+k}\right]\mathbb{E}\left[b^{(>d)}_{(l-1)r_{d}+l}\right]\right)^{-1},\\
\mu_{\pmb{\mathscr{G}}_{k,j_{d},l}^{(d)}}^{(t)}&=\nu_{\pmb{\mathscr{G}}_{k,j_{d},l}^{(d)}}^{(t)}\mu_{\pmb{\mathscr{G}}_{k,j_{d},l}^{(d)}}^{(t-1)}{\nu_{\pmb{\mathscr{G}}_{k,j_{d},l}^{(d)}}^{(t-1)}}^{-1}+\nu_{\pmb{\mathscr{G}}_{k,j_{d},l}^{(d)}}^{(t)}\mathbb{E}\left[\tau\right]\sum_{\mathbf{j}\in\Omega_{t}^{j_{d}}}\Bigg(x_{\mathbf{j}}\mathbb{E}\left[e^{(<d)}_{k}\right] \mathbb{E}\left[e^{(>d)}_{l}\right] \\&
-\sum_{k^{\prime}=1}^{r_{d-1}}\sum_{l^{\prime}=1}^{r_{d}} \mathbb{E}\left[b^{(<d)}_{(k-1)r_{d-1}+k^{\prime}}\right]\mathbb{E}\left[b^{(>d)}_{(l-1)r_{d}+l^{\prime}}\right]\mu_{\pmb{\mathscr{G}}^{(d)}_{k^{\prime},j_{d},l^{\prime}}}^{(t-1)}+ 
  \mathbb{E}\left[b^{(<d)}_{(k-1)r_{d-1}+k}\right]\mathbb{E}\left[b^{(>d)}_{(l-1)r_{d}+l}\right]\mu_{\pmb{\mathscr{G}}^{(d)}_{k,j_{d},l}}^{(t-1)}\Bigg).
\end{aligned}
\end{equation}

\section{The variational posterior distribution of the noise precision at time $t$.}\label{appendix3}
\begin{equation}
\begin{aligned}
    &\log q^{(t)}_{\tau}(\tau) = \mathbb{E}_{q^{(t)}(\mathbf{\theta}\setminus \tau)}\left[\log \tilde{p}^{(t)}(\mathbf{\theta})\right]+\operatorname{const}\\
    &=\mathbb{E}\left[ (\alpha^{(t-1)}-1)\log \tau-\beta^{(t-1)}\tau-\frac{\tau}{2}\sum_{\mathbf{j}\in\Omega_{t}}\left((x_{\mathbf{j}}-\prod_{d=1}^{D} \pmb{\mathscr{G}}^{(d)}_{j_{d}}\right)^{2}+ \frac{|\Omega_{t}|}{2}\log \tau\right]+\operatorname{const}\\
    &=\mathbb{E}\left[ (\alpha^{(t-1)}+\frac{|\Omega_{t}|}{2}-1)\log \tau-\left(\beta^{(t-1)}+\frac{1}{2}\sum_{\mathbf{j}\in\Omega_{t}}\left(x_{\mathbf{j}}-\prod_{d=1}^{D} \pmb{\mathscr{G}}^{(d)}_{j_{d}}\right)^{2}\right)\tau\right]+\operatorname{const}\\
    &= (\alpha^{(t-1)}+\frac{|\Omega_{t}|}{2}-1)\log \tau-\left(\beta^{(t-1)}+\frac{1}{2}\mathbb{E}\left[\sum_{\mathbf{j}\in\Omega_{t}}\left(x_{\mathbf{j}}-\prod_{d=1}^{D} \pmb{\mathscr{G}}^{(d)}_{j_{d}}\right)^{2}\right]\right)\tau+\operatorname{const}\\
    \end{aligned}
\end{equation}
Thus, $q^{(t)}_{\tau}(\tau)$ is also a Gamma distribution
\begin{equation} 
q^{(t)}_{\tau}(\tau) = \operatorname{Gam}(\tau|\alpha^{(t)},\beta^{(t)}),
\end{equation}
and the posterior hyperparameters can be updated by
\begin{equation}
\begin{aligned}
    \alpha^{(t)} &=\alpha^{(t-1)}+ \frac{|\Omega_{t}|}{2},\\
    \beta^{(t)} &=\beta^{(t-1)}+\frac{1}{2}\mathbb{E}\left[\sum_{\mathbf{j}\in\Omega_{t}}\left(x_{\mathbf{j}}-\prod_{d=1}^{D} \pmb{\mathscr{G}}^{(d)}_{j_{d}}\right)^{2}\right].
    \end{aligned}
\end{equation}
\section{The log of $\tilde{p}^{(t)}(\mathbf{\theta})$}
\label{appendix1}
Since $\tilde{p}^{(t)}(\mathbf{\theta}) = q^{(t-1)}(\mathbf{\theta})p(x_{\mathbf{j}\in\Omega_{t}}|\theta)$, where $\mathbf{\theta} = \{\{\pmb{\mathscr{G}}^{(d)}\}_{d=1}^{D},\tau\}$, the log of  $\tilde{p}^{(t)}(\mathbf{\theta})$ can be written as
\begin{equation}
\begin{aligned}
     \log \tilde{p}^{(t)}(\mathbf{\theta})&=\log\left( q^{(t-1)}(\mathbf{\theta})p(B_{t}|\mathbf{\theta})\right) \\
     &= \log q^{(t-1)}_{\tau}(\tau)+\sum_{d=1}^{D}\sum_{k=1}^{r_{d-1}} \sum_{l=1}^{r_{d}} \sum_{j_{d}=1}^{N_{d}} \log q^{(t-1)}_{\pmb{\mathscr{G}}^{(d)}_{k,j_{d},l}}\left(\pmb{\mathscr{G}}^{(d)}_{k,j_{d},l}\right)\log p(x_{\mathbf{j}\in\Omega_{t}}|\{\pmb{\mathscr{G}}^{(d)}\}_{d=1}^{D},\tau)   \\
     &= (\alpha-1)\log \tau-\beta\tau- \sum_{d=1}^{D}\sum_{k=1}^{N_{d-1}}\sum_{l=1}^{N_{d}}\sum_{j_{d}=1}^{N_{d}}\frac{\left(\pmb{\mathscr{G}}^{(d)}_{k,j_{d},l}-\mu_{\pmb{\mathscr{G}}_{k,j_{d},l}^{(d)}}^{(t-1)}\right)^{2}}{2\nu_{\pmb{\mathscr{G}}_{k,j_{d},l}^{(d)}}^{(t-1)}}\\
     &-\frac{\tau}{2}\sum_{\mathbf{j}\in\Omega_{t}}\left(x_{\mathbf{j}}-\prod_{d=1}^{D} \pmb{\mathscr{G}}^{(d)}_{j_{d}}\right)^{2}+ \frac{|\Omega_{t}|}{2}\log \tau.
\end{aligned}
\end{equation}
\end{document}